\def\year{2017}\relax
\documentclass[letterpaper]{article} 
\usepackage{aaai18}  
\usepackage{times}  
\usepackage{helvet}  
\usepackage{courier}  
\usepackage{url}  
\usepackage{graphicx}  
\usepackage{amsmath}
\usepackage{amssymb}
\usepackage{amsthm}
\usepackage{multirow}
\usepackage{tikz}
\usepackage{comment}

\usepackage{graphicx}
\usepackage{caption}
\usepackage{subcaption}
\usepackage{listings}
\usepackage{multicol}
\usepackage{arydshln}
\usetikzlibrary{calc,backgrounds,positioning,fit}

\newcommand{\tup}[1]{{\langle #1 \rangle}}

\newcommand{\pre}{\mathsf{pre}}     
\newcommand{\eff}{\mathsf{eff}}     
\newcommand{\cond}{\mathsf{cond}}   
\newcommand{\strips}{\textsc{Strips}}     

\newtheorem{theorem}{Theorem}
\newtheorem{lemma}[theorem]{Lemma}

\begin{document}

\title{Learning \strips\ Action Models with Classical Planning}
\author{\#39}

\author{Diego Aineto\and Sergio Jim\'enez\and Eva Onaindia\\
{\small Departamento de Sistemas Inform\'aticos y Computaci\'on}\\
{\small Universitat Polit\`ecnica de Val\`encia.}\\
{\small Camino de Vera s/n. 46022 Valencia, Spain}\\
{\small \{dieaigar,serjice,onaindia\}@dsic.upv.es}}

\maketitle
\begin{abstract}
This paper presents a novel approach for learning \strips\ action models from examples that compiles this inductive learning task into a classical planning task. Interestingly, the compilation approach is flexible to different amounts of available input knowledge; the learning examples can range from a set of plans (with their corresponding initial and final states) to just a pair of initial and final states (no intermediate action or state is given). Moreover, the compilation accepts partially specified action models and it can be used to validate whether the observation of a plan execution follows a given \strips\ action model, even if this model is not fully specified.
\end{abstract}

\section{Introduction}
Besides {\em plan synthesis}~\cite{ghallab2004automated}, planning action models are also useful for {\em plan/goal recognition}~\cite{ramirez2012plan}. In both planning tasks, an automated planner is required to reason about action models that correctly and completely capture the possible world transitions~\cite{geffner:book:2013}. Unfortunately, building planning action models is complex, even for planning experts, and this knowledge acquisition task is a bottleneck that limits the potential of AI planning~\cite{kambhampati:modellite:AAAI2007}.

On the other hand, Machine Learning (ML) has shown to be able to compute a wide range of different kinds of models from examples~\cite{michalski2013machine}. The application of inductive ML to learning \strips\ action models, the vanilla action model for planning~\cite{fikes1971strips}, is not straightforward though:
\begin{itemize}
\item The {\em input} to ML algorithms (the learning/training data) is usually a finite set of vectors that represent the value of some fixed object features. The input for learning planning action models is, however, observations of plan executions (where each plan possibly has a different length).
\item The {\em output} of ML algorithms is usually a scalar value (an integer, in the case of {\em classification} tasks, or a real value, in the case of {\em regression} tasks). When learning action models the output is, for each action, the preconditions, negative and positive effects that define the possible state transitions.
\end{itemize}

Learning \strips\ action models is a well-studied problem with sophisticated algorithms such as {\sc ARMS} \cite{yang2007learning}, {\sc SLAF} \cite{amir:alearning:JAIR08} or {\sc LOCM} \cite{cresswell2013acquiring}, which do not require full knowledge of the intermediate states traversed by the example plans. Motivated by recent advances on the synthesis of different kinds of generative models with classical planning \cite{bonet2009automatic,segovia2016hierarchical,segovia2017generating}, this paper introduces an innovative planning compilation approach for learning \strips\ action models. The compilation approach is appealing by itself because it opens up the door to the bootstrapping of planning action models, but also because:

\begin{enumerate}
\item It is flexible to various amounts of input knowledge. Learning examples range from a set of plans (with their corresponding initial and final states) to just a pair of initial and final states where no intermediate state or action is observed.
\item It accepts previous knowledge about the structure of the actions in the form of partially specified action models. In the extreme, the compilation can validate whether an observed plan execution is valid for a given \strips\ action model, even if this model is not fully specified.
\end{enumerate}

The second section of the paper formalizes the classical planning model, its extension to {\em conditional effects} (a requirement of the proposed compilation) and the \strips\ action model (the output of the addressed learning task). The third section formalizes the task of learning action models with different amounts of available input knowledge. The fourth and fifth sections describe our compilation approach to tackle the formalized learning tasks. Finally, the last sections show the experimental evaluation, discuss the strengths and weaknesses of the compilation approach and propose several opportunities for future research.

\section{Background}
This section defines the planning model and the output of the learning tasks addressed in the paper.

\subsection{Classical planning with conditional effects}
Our approach to learning \strips\ action models is compiling this learning task into a classical planning task with conditional effects. Conditional effects allow us to compactly define actions whose effects depend on the current state. Supporting conditional effects is now a requirement of the IPC~\cite{vallati:IPC:AIM2015} and many classical planners cope with conditional effects without compiling them away.

We use $F$ to denote the set of {\em fluents} (propositional variables) describing a state. A {\em literal} $l$ is a valuation of a fluent $f\in F$; i.e. either~$l=f$ or $l=\neg f$. A set of literals $L$ represents a partial assignment of values to fluents (without loss of generality, we will assume that $L$ does not contain conflicting values). We use $\mathcal{L}(F)$ to denote the set of all literal sets on $F$; i.e.~all partial assignments of values to fluents.

A {\em state} $s$ is a full assignment of values to fluents; $|s|=|F|$, so the size of the state space is $2^{|F|}$. Explicitly including negative literals $\neg f$ in states simplifies subsequent definitions but often we will abuse of notation by defining a state $s$ only in terms of the fluents that are true in $s$, as it is common in \strips\ planning.

A {\em classical planning frame} is a tuple $\Phi=\tup{F,A}$, where $F$ is a set of fluents and $A$ is a set of actions. An action $a\in A$ is defined with {\em preconditions}, $\pre(a)\subseteq\mathcal{L}(F)$, {\em positive effects}, $\eff^+(a)\subseteq\mathcal{L}(F)$, and {\em negative effects} $\eff^-(a)\subseteq\mathcal{L}(F)$. We say that an action $a\in A$ is {\em applicable} in a state $s$ iff $\pre(a)\subseteq s$. The result of applying $a$ in $s$ is the {\em successor state} denoted by $\theta(s,a)=\{s\setminus\eff^-(a))\cup\eff^+(a)\}$.

An action $a\in A$ with conditional effects is defined as a set of {\em preconditions} $\pre(a)$ and a set of {\em conditional effects} $\cond(a)$. Each conditional effect $C\rhd E\in\cond(a)$ is composed of two sets of literals: $C \subseteq \mathcal{L}(F)$, the {\em condition}, and $E \subseteq \mathcal{L}(F)$, the {\em effect}. An action $a\in A$ is {\em applicable} in a state $s$ iff $\pre(a)\subseteq s$, and the {\em triggered effects} resulting from the action application are the effects whose conditions hold in $s$:
\[
triggered(s,a)=\bigcup_{C\rhd E\in\cond(a),C\subseteq s} E,
\]

The result of applying action $a$ in state $s$ is the {\em successor} state $\theta(s,a)=\{s\setminus\eff_c^-(s,a))\cup\eff_c^+(s,a)\}$ where $\eff_c^-(s,a)\subseteq triggered(s,a)$ and $\eff_c^+(s,a)\subseteq triggered(s,a)$ are, respectively, the triggered {\em negative} and {\em positive} effects.

A {\em classical planning problem} is a tuple $P=\tup{F,A,I,G}$, where $I$ is an initial state and $G\subseteq\mathcal{L}(F)$ is a goal condition. A {\em plan} for $P$ is an action sequence $\pi=\tup{a_1, \ldots, a_n}$ that induces the {\em state trajectory} $\tup{s_0, s_1, \ldots, s_n}$ such that $s_0=I$ and $a_i$ ({\small $1\leq i\leq n$}) is applicable in $s_{i-1}$ and generates the successor state $s_i=\theta(s_{i-1},a_i)$. The {\em plan length} is denoted with $|\pi|=n$ . A plan $\pi$ {\em solves} $P$ iff $G\subseteq s_n$; i.e.~if the goal condition is satisfied in the last state resulting from the application of the plan $\pi$ in the initial state $I$.

\subsection{\strips\ action schemas and {\em variable name} objects}
Our approach is aimed at learning PDDL action schemas that follow the \strips\ requirement~\cite{mcdermott1998pddl,fox2003pddl2}. Figure~\ref{fig:stack} shows the {\em stack} schema of a four-operator {\em blocksworld} domain~\cite{slaney2001blocks} encoded in PDDL.

\begin{figure}
\begin{footnotesize}
\begin{verbatim}
(:action stack
  :parameters (?v1 ?v2 - object)
  :precondition (and (holding ?v1) (clear ?v2))
  :effect (and (not (holding ?v1))
               (not (clear ?v2))
               (handempty) (clear ?v1)
               (on ?v1 ?v2)))
\end{verbatim}
\end{footnotesize}
 \caption{\small The {\em stack} operator schema of the {\em blocksworld} domain specified in PDDL.}
\label{fig:stack}
\end{figure}

To formalize the output of the learning task, we assume that fluents $F$ are instantiated from a set of {\em predicates} $\Psi$, as in PDDL. Each predicate $p\in\Psi$ has an argument list of arity $ar(p)$. Given a set of {\em objects} $\Omega$, the set of fluents $F$ is induced by assigning objects in $\Omega$ to the arguments of the predicates in $\Psi$; i.e.~$F=\{p(\omega):p\in\Psi,\omega\in\Omega^{ar(p)}\}$, where $\Omega^k$ is the $k$-th Cartesian power of $\Omega$.

Let $\Omega_v=\{v_i\}_{i=1}^{\operatorname*{max}_{a\in A} ar(a)}$ be a new set of objects denoted as {\em variable names} ($\Omega\cap\Omega_v=\emptyset$). $\Omega_v$ is bound to the maximum arity of an action in a given planning frame. For instance, in a three-block \emph{blocksworld}, $\Omega=\{block_1, block_2, block_3\}$ while $\Omega_v=\{v_1, v_2\}$ because the operators with the maximum arity, {\small\tt stack} and {\small\tt unstack}, have two parameters each.

Let $F_v$ be a new set of fluents, $F\cap F_v=\emptyset$, that results from instantiating the predicates in $\Psi$ using exclusively  objects of $\Omega_v$. $F_v$ defines the elements of the preconditions and effects of an action schema. For instance, in the \emph{blocksworld} domain, $F_v$={\small\tt\{handempty, holding($v_1$), holding($v_2$), clear($v_1$), clear($v_2$), ontable($v_1$), ontable($v_2$), on($v_1,v_1$), on($v_1,v_2$), on($v_2,v_1$), on($v_2,v_2$)\}}.

Finally, we assume that an action $a\in A$ is instantiated from a \strips\ operator schema $\xi=\tup{head(\xi),pre(\xi),add(\xi),del(\xi)}$ where:

\begin{itemize}
\item $head(\xi)=\tup{name(\xi),pars(\xi)}$ is the operator {\em header} defined by its name and the corresponding {\em variable names}, $pars(\xi)=\{v_i\}_{i=1}^{ar(\xi)}$. For instance, the headers of a four-operator \emph{blocksworld} domain are: {\small\tt pickup($v_1$), putdown($v_1$), stack($v_1,v_2$)} and {\small\tt unstack($v_1,v_2$)}.
\item $pre(\xi)\subseteq F_v$ is the set of preconditions,  $del(\xi)\subseteq F_v$ the negative effects and  $add(\xi)\subseteq F_v$ the positive effects such that $del(\xi)\subseteq pre(\xi)$, $del(\xi)\cap add(\xi)=\emptyset$ and $pre(\xi)\cap add(\xi)=\emptyset$.
\end{itemize}

\section{Learning \strips\ action models}

Learning \strips\ action models from fully available input knowledge, i.e. from plans where the {\em pre-} and {\em post-states} of every action in the plans are known, is straightforward. When intermediate states are available, operator schemas are derived lifting the literals that change between the pre and post-state of each action execution. Preconditions of an action are derived lifting the minimal set of literals that appears in all the pre-states of the corresponding action~\cite{jimenez2012review}.

This section formalizes more challenging learning tasks, where less input knowledge is available:

\subsubsection{Learning from (initial, final) state pairs.} This learning task amounts to observing agents acting in the world but watching only the result of their plans execution. No intermediate information about the actions in the plans is given. This learning task is formalized as $\Lambda=\tup{\Psi,\Sigma}$:
\begin{itemize}
\item $\Psi$ is the set of predicates that define the abstract state space of a given planning domain.
\item $\Sigma=\{\sigma_1,\ldots,\sigma_{\tau}\}$ is a set of $(initial, final)$ state pairs called {\em labels}. Each label $\sigma_t=(s_0^t,s_{n}^t)$, {\tt\small $1\leq t\leq \tau$}, comprises the {\em final} state $s_{n}^t$ resulting from executing an unknown plan $\pi_t=\tup{a_1^t, \ldots, a_n^t}$ in the {\em initial} state $s_0^t$.
\end{itemize}

\subsubsection{Learning from labeled plans.}
We augment the input knowledge with the actions executed by the observed agent and define the learning task $\Lambda'=\tup{\Psi,\Sigma,\Pi}$:

\begin{itemize}
\item $\Pi=\{\pi_1,\ldots,\pi_{\tau}\}$ is a given set of example plans where $\pi_t=\tup{a_1^t, \ldots, a_n^t}$, {\small $1\leq t\leq \tau$}, is an action sequence that induces the corresponding state sequence $\tup{s_0^t, s_1^t, \ldots, s_n^t}$ such that $a_i^t$, {\small $1\leq i\leq n$}, is applicable in $s_{i-1}^t$ and generates $s_i^t=\theta(s_{i-1}^t,a_i^t)$.
\end{itemize}

Figure~\ref{fig:lexample} shows an example of a learning task $\Lambda'$ of the {\em blocksworld} domain. This task has a single learning example, $\Pi=\{\pi_1\}$ and $\Sigma=\{\sigma_1\}$, that corresponds to observing the execution of an eight-action plan $(|\pi_1|=8)$ for inverting a four-block tower.

\subsubsection{Learning from partially specified action models.}
In case that partially specified operator schemas are available, we can incorporate this information within the learning task. The new leaning task is defined as $\Lambda''=\tup{\Psi,\Sigma,\Pi,\Xi_0}$:

\begin{itemize}
\item $\Xi_0$ is a partially specified action model in which some preconditions and effects are known a priori.
\end{itemize}

\vspace{0.1cm}

A solution to $\Lambda$ is a set of operator schemas $\Xi$ that is compliant with the predicates in $\Psi$ and the set of initial and final states $\Sigma$. In a $\Lambda$ learning scenario, a solution must not only determine a possible \strips\ action model but also the plans $\pi_t$, {\tt\small $1\leq t\leq \tau$}, that explain the given labels $\Sigma$ using the learned model. A solution to $\Lambda'$ is a set of \strips\ operator schemas $\Xi$ (one schema $\xi=\tup{head(\xi),pre(\xi),add(\xi),del(\xi)}$ for each action that has a different name in the example plans $\Pi$) that is compliant with the predicates $\Psi$, the example plans $\Pi$, and their corresponding labels $\Sigma$.  Finally, a solution to $\Lambda''$ is a set of \strips\ operator schemas $\Xi$ that is also compliant with the provided partially specified action model $\Xi_0$.

\begin{figure}
{\tt ;;; Predicates in $\Psi$}
\begin{footnotesize}
\begin{verbatim}
(handempty) (holding ?o  - object)
(clear ?o - object) (ontable ?o - object)
(on ?o1 - object ?o2 - object)
\end{verbatim}
\end{footnotesize}

\vspace{0.2cm}

\begin{subfigure}{.25\textwidth}
{\tt ;;; Plan $\pi_1$}
\begin{footnotesize}
\begin{verbatim}
0: (unstack A B)
1: (putdown A)
2: (unstack B C)
3: (stack B A)
4: (unstack C D)
5: (stack C B)
6: (pickup D)
7: (stack D C)
\end{verbatim}
\end{footnotesize}
\end{subfigure}%
\begin{subfigure}{.6\textwidth}
{\tt ;;; Label $\sigma_1=(s_0^1,s_{n}^1)$}
\begin{lstlisting}[mathescape]
\end{lstlisting}
\vspace{0.1cm}
\begin{tikzpicture}[node distance = 0mm, block/.style args = {#1,#2}{fill=#1,text width=#2,shape=square}]
\node (initD) [draw]{D};
\node (initC) [draw, above=of initD.north]{C};
\node (initB) [draw, above=of initC.north]{B};
\node (initA) [draw, above=of initB.north]{A};
\draw[thick] (-1,-0.25) -- (2.5,-0.25);

\node (goalA) [draw, right=10mm of initD]{A};
\node (goalB) [draw, right=10mm of initC]{B};
\node (goalC) [draw, right=10mm of initB]{C};
\node (goalD) [draw, right=10mm of initA]{D};
\end{tikzpicture}
\vspace{0.6cm}
\end{subfigure}%
 \caption{\small Learning task of the \emph{blocksworld} domain from a single labeled plan.}
\label{fig:lexample}
\end{figure}
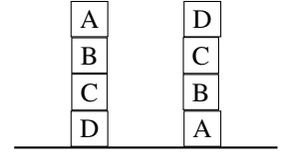

\section{Learning \strips\ action models with planning}
In our approach, a learning task $\Lambda$, $\Lambda'$ or $\Lambda''$ is solved by compiling it into a classical planning task with conditional effects. The intuition behind the compilation is that a solution to the resulting classical planning task is a sequence of actions that:

\begin{enumerate}
\item {\em Programs the \strips\ action model $\Xi$}. A solution plan has a {\em prefix} that, for each $\xi\in\Xi$, determines the fluents from $F_v$ that belong to $pre(\xi)$, $del(\xi)$ and $add(\xi)$.
\item {\em Validates the programmed \strips\ action model $\Xi$ in the given input knowledge} (the labels $\Sigma$ and $\Pi$, and/or $\Xi_0$ if available). For every label $\sigma_t\in \Sigma$, a solution plan has a postfix that produces a final state $s_{n}^t$ using the programmed action model $\Xi$ in the corresponding initial state $s_0^t$. This process is the validation of the programmed action model $\Xi$ with the set of learning examples {\small $1\leq t\leq \tau$}. 
\end{enumerate}

To formalize our compilation we first define a set of classical planning instances $P_t=\tup{F,\emptyset,I_t,G_t}$ that belong to the same planning frame (i.e. same fluents and actions but different initial states and goals). Fluents $F$ are built instantiating the predicates in $\Psi$ with the objects of the input labels $\Sigma$. Formally, $\Omega=\bigcup_{\small 1\leq t\leq \tau} obj(s_0^t)$, where $obj$ is a function that returns the objects that appear in a fully specified state. The set of actions, $A=\emptyset$, is empty because the action model is initially unknown. Finally, the initial state $I_t$ is given by the state $s_0^t\in \sigma_t$, and the goals $G_t$ are defined by the state $s_n^t\in \sigma_t$.

We can now formalize the compilation approach. We start with $\Lambda$ as it requires the least input knowledge. Given a learning task $\Lambda=\tup{\Psi,\Sigma}$, the compilation outputs a classical planning task $P_{\Lambda}=\tup{F_{\Lambda},A_{\Lambda},I_{\Lambda},G_{\Lambda}}$:

\begin{itemize}
\item $F_{\Lambda}$ extends $F$ with:
\begin{itemize}
\item Fluents $pre_f(\xi)$, $del_f(\xi)$ and $add_f(\xi)$, for every $f\in F_v$ and $\xi \in \Xi$ that represent the programmed action model. If a fluent of $pre_f(\xi)/del_f(\xi)/add_f(\xi)$ holds, it means that $f$ is a precondition/negative effect/positive effect of the operator schema $\xi\in \Xi$. For instance, the preconditions of the $stack$ schema (Figure~\ref{fig:stack}) are represented by fluents {\small\tt pre\_holding\_stack\_$v_1$} and {\small\tt pre\_clear\_stack\_$v_2$}.
\item A fluent $mode_{prog}$ to indicate whether the operator schemas are being programmed or validated (when already programmed)
\item Fluents $\{test_t\}_{1\leq t\leq \tau}$ which represent the examples where the action model will be validated.
\end{itemize}

\item $I_{\Lambda}$ contains the fluents from $F$ that encode $s_0^1$ (the initial state of the first label), the fluents in every $pre_f(\xi)\in F_{\Lambda}$ and the fluent $mode_{prog}$ set to true. Our compilation assumes that any operator schema is initially programmed with every possible precondition (the most specific learning hypothesis), no negative effect and no positive effect.
\item $G_{\Lambda}=\bigcup_{1\leq t\leq \tau}\{test_t\}$ indicates that the programmed action model is validated in all the learning examples.
\item $A_{\Lambda}$ contains actions of three kinds:
\begin{enumerate}
\item Actions for {\em programming} an operator schema $\xi\in\Xi$:
\begin{itemize}
\item Actions for {\bf removing} a {\em precondition} $f\in F_v$ from $\xi$.

\begin{small}
\begin{align*}
\hspace*{7pt}\pre(\mathsf{programPre_{f,\xi}})=&\{\neg del_{f}(\xi),\neg add_{f}(\xi),\\
& mode_{prog}, pre_{f}(\xi)\},\\
\cond(\mathsf{programPre_{f,\xi}})=&\{\emptyset\}\rhd\{\neg pre_{f}(\xi)\}.
\end{align*}
\end{small}

\item Actions for {\bf adding} a {\em negative} or {\em positive} effect $f\in F_v$ to $\xi$.

\begin{small}
\begin{align*}
\hspace*{7pt}\pre(\mathsf{programEff_{f,\xi}})=&\{\neg del_{f}(\xi),\neg add_{f}(\xi),\\
& mode_{prog}\},\\
\cond(\mathsf{programEff_{f,\xi}})=&\{pre_{f}(\xi)\}\rhd\{del_{f}(\xi)\},\\
&\{\neg pre_{f}(\xi)\}\rhd\{add_{f}(\xi)\}.
\end{align*}
\end{small}
\end{itemize}

\item Actions for {\em applying} an already programmed operator schema $\xi\in\Xi$ bound to the objects $\omega\subseteq\Omega^{ar(\xi)}$. We assume operators headers are known so the binding of $\xi$ is done implicitly by order of appearance of the action parameters, i.e. variables $pars(\xi)$ are bound to the objects in $\omega$ that appear in the same position. Figure~\ref{fig:compilation} shows the PDDL encoding of the action for applying a programmed operator $stack$.
\begin{small}
\begin{align*}
\hspace*{7pt}\pre(\mathsf{apply_{\xi,\omega}})=&\{pre_{f}(\xi)\implies p(\omega)\}_{\forall p\in\Psi,f=p(pars(\xi))},\\
\cond(\mathsf{apply_{\xi,\omega}})=&\{del_{f}(\xi)\}\rhd\{\neg p(\omega)\}_{\forall p\in\Psi,f=p(pars(\xi))},\\
&\{add_{f}(\xi)\}\rhd\{p(\omega)\}_{\forall p\in\Psi,f=p(pars(\xi))},\\
&\{mode_{prog}\}\rhd\{\neg mode_{prog}\}.
\end{align*}
\end{small}

\item Actions for {\em validating} the learning example {\tt\small $1\leq t\leq \tau$}.
\begin{small}
\begin{align*}
\hspace*{7pt}\pre(\mathsf{validate_{t}})=&G_t\cup\{test_j\}_{1\leq j<t}\\
&\cup\{\neg test_j\}_{t\leq j\leq \tau}\cup \{\neg mode_{prog}\},\\
\cond(\mathsf{validate_{t}})=&\{\emptyset\}\rhd\{test_t\} \cup \{\neg f\}_{\forall f\in G_t, f \notin I_{t+1}}\\
&\cup \{f\}_{\forall f\in I_{t+1}, f \notin G_t}.
\end{align*}
\end{small}
\end{enumerate}
\end{itemize}

\begin{lemma}
A classical plan $\pi$ that solves $P_{\Lambda}$ induces an action model $\Xi$ that solves the learning task $\Lambda$.
\end{lemma}

\begin{proof}[Proof sketch]
\begin{small}
Once operator schemas $\Xi$ are programmed, they can only be applied and validated according to the $mode_{prog}$ fluent. To solve $P_{\Lambda}$, goals $\{test_t\}$, {\small $1\leq t\leq \tau$} can only be achieved by executing an applicable sequence of programmed operator schemas that reaches the final state $s_n^t$, defined in $\sigma_t$, starting from $s_0^t$. If this is achieved for all the input examples {\small $1\leq t\leq \tau$}, it means that the programmed action model $\Xi$ is compliant with the provided input knowledge and hence it is a solution to $\Lambda$.
\end{small}
\end{proof}

The compilation is {\em complete} in the sense that it does not discard any possible \strips\ action model. The size of the classical planning task $P_{\Lambda}$ depends on:
\begin{itemize}
\item The arity of the actions headers in $\Xi$ and the predicates $\Psi$ of the learning task. The larger the arity, the larger the $F_v$ set, which in turn defines the size of the fluent sets $pre_f(\xi)/del_f(\xi)/add_f(\xi)$ and the corresponding set of {\em programming} actions.
\item The number of learning examples. The larger this number, the more $test_t$ fluents and $\mathsf{validate_{t}}$ actions in $P_{\Lambda}$.
\end{itemize}

\begin{figure}
\begin{scriptsize}
\begin{verbatim}
(:action apply_stack
  :parameters (?o1 - object ?o2 - object)
  :precondition
   (and (or (not (pre_on_stack_v1_v1)) (on ?o1 ?o1))
        (or (not (pre_on_stack_v1_v2)) (on ?o1 ?o2))
        (or (not (pre_on_stack_v2_v1)) (on ?o2 ?o1))
        (or (not (pre_on_stack_v2_v2)) (on ?o2 ?o2))
        (or (not (pre_ontable_stack_v1)) (ontable ?o1))
        (or (not (pre_ontable_stack_v2)) (ontable ?o2))
        (or (not (pre_clear_stack_v1)) (clear ?o1))
        (or (not (pre_clear_stack_v2)) (clear ?o2))
        (or (not (pre_holding_stack_v1)) (holding ?o1))
        (or (not (pre_holding_stack_v2)) (holding ?o2))
        (or (not (pre_handempty_stack)) (handempty)))
  :effect
   (and (when (del_on_stack_v1_v1) (not (on ?o1 ?o1)))
        (when (del_on_stack_v1_v2) (not (on ?o1 ?o2)))
        (when (del_on_stack_v2_v1) (not (on ?o2 ?o1)))
        (when (del_on_stack_v2_v2) (not (on ?o2 ?o2)))
        (when (del_ontable_stack_v1) (not (ontable ?o1)))
        (when (del_ontable_stack_v2) (not (ontable ?o2)))
        (when (del_clear_stack_v1) (not (clear ?o1)))
        (when (del_clear_stack_v2) (not (clear ?o2)))
        (when (del_holding_stack_v1) (not (holding ?o1)))
        (when (del_holding_stack_v2) (not (holding ?o2)))
        (when (del_handempty_stack) (not (handempty)))
        (when (add_on_stack_v1_v1) (on ?o1 ?o1))
        (when (add_on_stack_v1_v2) (on ?o1 ?o2))
        (when (add_on_stack_v2_v1) (on ?o2 ?o1))
        (when (add_on_stack_v2_v2) (on ?o2 ?o2))
        (when (add_ontable_stack_v1) (ontable ?o1))
        (when (add_ontable_stack_v2) (ontable ?o2))
        (when (add_clear_stack_v1) (clear ?o1))
        (when (add_clear_stack_v2) (clear ?o2))
        (when (add_holding_stack_v1) (holding ?o1))
        (when (add_holding_stack_v2) (holding ?o2))
        (when (add_handempty_stack) (handempty))
        (when (modeProg) (not (modeProg)))))
\end{verbatim}
\end{scriptsize}
 \caption{\small PDDL action for applying an already programmed schema $stack$ (implications coded as disjunctions).}
\label{fig:compilation}
\end{figure}

\section{Constraining the learning hypothesis space with additional input knowledge}
\label{sec:Constraining}
In this section, we show that further input knowledge can be used to constrain the space of possible action models and to make the learning task more practicable.

\subsection{Labeled plans}
We extend the compilation to consider labeled plans. Given a learning task $\Lambda'=\tup{\Psi,\Sigma,\Pi}$, the compilation outputs a classical planning task $P_{\Lambda'}=\tup{F_{\Lambda'},A_{\Lambda'},I_{\Lambda'},G_{\Lambda'}}$:
\begin{itemize}
\item $F_{\Lambda'}$ extends $F_{\Lambda}$ with $F_{\Pi}=\{plan(name(\xi),\Omega^{ar(\xi)},j)\}$, the fluents to code the steps of the plans in $\Pi$, where $F_{\pi_t}\subseteq F_{\Pi}$ encodes $\pi_t\in \Pi$. Fluents $at_j$ and $next_{j,j+1}$, {\small $1\leq j< n$}, are also added to represent the current plan step and to iterate through the steps of a plan.
\item $I_{\Lambda'}$ extends $I_{\Lambda}$ with fluents $F_{\pi_1}$ plus fluents $at_1$ and $\{next_{j,j+1}\}$, {\small $1\leq j<n$}, to indicate the plan step where the action model is validated. As in the original compilation, $G_{\Lambda'}=G_{\Lambda}=\bigcup_{1\leq t\leq \tau}\{test_t\}$.
\item With respect to $A_{\Lambda'}$.
\begin{enumerate}
\item The actions for {\em programming} the preconditions/effects of a given operator schema $\xi\in\Xi$ are the same.
\item The actions for {\em applying} an already programmed operator have an extra precondition $f\in F_{\Pi}$ that encodes the current plan step, and extra conditional effects $\{at_{j}\}\rhd\{\neg at_{j},at_{j+1}\}_{\forall j\in [1,n]}$ for advancing to the next plan step. With this mechanism we ensure that these actions are applied in the same order as in the example plans.
\item The actions for {\em validating} the current learning example have an extra precondition, $at_{|\pi_t|}$, to indicate that the current plan $\pi_t$ is fully executed and extra conditional effects to remove plan $\pi_{t}$ and initiate the next plan $\pi_{t+1}$:
\begin{small}
\begin{align*}
&\{\emptyset\}\rhd\{\neg at_{|\pi_t|},at_1\} \cup \{\neg f\}_{f\in F_{\pi_t}} \cup \{f\}_{f\in F_{\pi_t+1}}.
\end{align*}
\end{small}
\end{enumerate}
\end{itemize}

\subsection{Partially specified action models}
The known preconditions and effects of a partially specified action model are encoded as fluents $pre_f(\xi)$, $del_f(\xi)$ and $add_f(\xi)$ set to true in the initial state $I_{\Lambda'}$. The programming actions, $\mathsf{programPre_{f,\xi}}$ and $\mathsf{programEff_{f,\xi}}$, become now unnecessary and they are removed from $A_{\Lambda'}$, thus making the planning task $P_{\Lambda'}$ be easier to solve.

To illustrate this, the plan of Figure~\ref{fig:plan} is a solution to a learning task $\Lambda''=\tup{\Psi,\Sigma,\Pi,\Xi_0}$ for acquiring the {\em blocksworld} action model where operator schemas for {\tt\small pickup}, {\tt\small putdown} and {\tt\small unstack} are specified in $\Xi_0$. This plan programs and validates the operator schema {\tt\small stack} from {\em blocksworld} using the plan $\pi_1$ and label $\sigma_1$ shown in Figure~\ref{fig:lexample}. Plan steps $[0,8]$ program the preconditions of the {\tt\small stack} operator, steps $[9,13]$ program the operator effects and steps $[14,22]$ validate the programmed operators following the plan $\pi_1$ shown in Figure~\ref{fig:lexample}.

In the extreme, when a fully specified \strips\ action model is given in $\Xi_0$, the compilation validates whether an observed plan follows the given model. In this case, if a solution plan is found for $P_{\Lambda'}$, it means that the given action model is {\em valid} for the provided examples. If $P_{\Lambda'}$ is unsolvable then it means that the action model is invalid because it is not compliant with all the given examples. Tools for plan validation like VAL~\cite{howey2004val} could also be used at this point.

\begin{figure}
{\footnotesize\tt
     {\bf 00} : (program\_pre\_clear\_stack\_v1)\\
     01 : (program\_pre\_handempty\_stack)\\
     02 : (program\_pre\_holding\_stack\_v2)\\
     03 : (program\_pre\_on\_stack\_v1\_v1)\\
     04 : (program\_pre\_on\_stack\_v1\_v2)\\
     05 : (program\_pre\_on\_stack\_v2\_v1)\\
     06 : (program\_pre\_on\_stack\_v2\_v2)\\
     07 : (program\_pre\_ontable\_stack\_v1)\\
     08 : (program\_pre\_ontable\_stack\_v2)\\
     {\bf 09} : (program\_eff\_clear\_stack\_v1)\\
    10 : (program\_eff\_clear\_stack\_v2)\\
    11 : (program\_eff\_handempty\_stack)\\
    12 : (program\_eff\_holding\_stack\_v1)\\
    13 : (program\_eff\_on\_stack\_v1\_v2)\\
    {\bf 14} : (apply\_unstack a b i1 i2)\\
    15 : (apply\_putdown a i2 i3)\\
    16 : (apply\_unstack b c i3 i4)\\
    17 : (apply\_stack b a i4 i5)\\
    18 : (apply\_unstack c d i5 i6)\\
    19 : (apply\_stack c b i6 i7)\\
    20 : (apply\_pickup d i7 i8)\\
    21 : (apply\_stack d c i8 i9)\\
    {\bf 22} : (validate\_1)
}
 \caption{\small Plan for programming and validating the $stack$ schema using plan $\pi_1$ and label $\sigma_1$ (shown in Figure~\ref{fig:lexample}) as well as previously specified operator schemas for $pickup$, $putdown$ and $unstack$.}
\label{fig:plan}
\end{figure}

\subsection{Static predicates}
A {\em static predicate} $p \in \Psi$ is a predicate that does not appear in the effects of any action~\cite{fox:TIM:JAIR1998}. Therefore, one can get rid of the mechanism for programming these predicates in the effects of any action schema while keeping the compilation complete. Given a static predicate $p$:
\begin{itemize}
\item Fluents $del_f(\xi)$ and $add_f(\xi)$, such that $f\in F_v$ is an instantiation of the static predicate $p$ in the set of {\em variable names} $\Omega_v$, can be discarded for every $\xi\in\Xi$.
\item Actions $\mathsf{programEff_{f,\xi}}$ (s.t. $f\in F_v$ is an instantiation of $p$ in $\Omega_v$) can also be discarded for every $\xi\in\Xi$.
\end{itemize}

Static predicates can also constrain the space of possible preconditions by looking at the given set of labels $\Sigma$. One can assume that if a precondition $f\in F_v$ (s.t. $f\in F_v$ is an instantiation of a static predicate in $\Omega_v$) is not compliant with the labels in $\Sigma$ then fluents $pre_f(\xi)$ and actions $\mathsf{programPre_{f,\xi}}$ can be discarded for every $\xi\in\Xi$. For instance, in the {\em zenotravel} domain, $pre\_next\_board\_v1\_v1$, $pre\_next\_debark\_v1\_v1$, $pre\_next\_fly\_v1\_v1$, $pre\_next\_zoom\_v1\_v1$, $pre\_next\_refuel\_v1\_v1$ can be discarded (and their corresponding programming actions) because a precondition {\tt\small(next ?v1 ?v1 - flevel)} will never hold in any state of $\Sigma$.

On the other hand, fluents $pre_f(\xi)$ and actions $\mathsf{programPre_{f,\xi}}$ are discardable for every $\xi\in\Xi$ if a precondition $f\in F_v$ (s.t. $f\in F_v$ is an instantiation of a static predicate in $\Omega_v$) is not possible according to $\Pi$. Back to the {\em zenotravel} domain, if an example plan $\pi_t\in \Pi$ contains the action {\tt\small (fly plane1 city2 city0 fl3 fl2)} and the corresponding label $\sigma_t\in\Sigma$ contains the static literal {\tt\small (next fl2 fl3)} but does not contain {\tt\small (next fl2 fl2)}, {\tt\small (next fl3 fl3)} or {\tt\small (next fl3 fl2)}, the only possible precondition that would include the static predicate is $pre\_next\_fly\_v5\_v4$.


\section{Evaluation}
This section evaluates the performance of our approach for learning \strips\ action models with different amounts of available input knowledge.

\subsubsection{Setup.}
The domains used in the evaluation are IPC domains that satisfy the \strips\ requirement~\cite{fox2003pddl2}, taken from the {\sc planning.domains} repository~\cite{muise2016planning}. We only used 5 learning examples for each domain and we fixed the examples for all the experiments so that we can evaluate the impact of the input knowledge in the quality of the learned models. All experiments are run on an Intel Core i5 3.10 GHz x 4 with 8 GB of RAM.

The classical planner we used to solve the instances that result from our compilations is {\sc Madagascar}~\cite{rintanen2014madagascar}. We used {\sc Madagascar} due to its ability to deal with planning instances populated with dead-ends. In addition, SAT-based planners can apply the actions for programming preconditions in a single planning step (in parallel) because these actions do not interact. Actions for programming action effects can also be applied in a single planning step reducing significantly the planning horizon.


\subsubsection{Metrics.}
The quality of the learned models is measured with the {\em precision} and {\em recall} metrics. These two metrics are frequently used in {\em pattern recognition}, {\em information retrieval} and {\em binary classification} and are more informative that simply counting the number of errors in the learned model or computing the {\em symmetric difference} between the learned and the reference model~\cite{davis2006relationship}.

Intuitively, precision gives a notion of {\em soundness} while recall gives a notion of the {\em completeness} of the learned models. Formally, $Precision=\frac{tp}{tp+fp}$, where $tp$ is the number of true positives (predicates that correctly appear in the action model) and $fp$ is the number of false positives (predicates appear in the learned action model that should not appear). Recall is formally defined as $Recall=\frac{tp}{tp+fn}$ where $fn$ is the number of false negatives (predicates that should appear in the learned action model but are missing).

Given the syntax-based nature of these metrics, it may happen that they report low scores for learned models that are actually good but correspond to {\em reformulations} of the actual model; i.e. a learned model semantically equivalent but syntactically different to the reference model. This mainly occurs when the learning task is under-constrained.

\subsection{Learning from labeled plans}
We start evaluating our approach with tasks $\Lambda'=\tup{\Psi,\Sigma,\Pi}$, where {\em labeled plans} are available. We then repeat the evaluation but exploiting potential \emph{static predicates} computed from $\Sigma$, which are the predicates that appear  unaltered in the initial and final states in every $\sigma_t\in\Sigma$. Static predicates are used to constrain the space of possible action models as explained in the previous section.

Table~\ref{tab:results_plans} shows the obtained results. Precision ({\bf P}) and recall ({\bf R}) are computed separately for the preconditions ({\bf Pre}), positive effects ({\bf Add}) and negative Effects ({\bf Del}), while the last two columns of each setting and the last row report averages values. We can observe that identifying static predicates leads to models with better precondition {\em recall}. This fact evidences that many of the missing preconditions corresponded to static predicates because there is no incentive to learn them as they always hold~\cite{gregory2015domain}.

Table~\ref{tab:time_plans} reports the total planning time, the preprocessing time (in seconds) invested by {\sc Madagascar} to solve the planning instances that result from our compilation as well as the number of actions of the solution plans. All the learning tasks are solved in a few seconds. Interestingly, one can identify the domains with static predicates by just looking at the reported plan length. In these domains some of the preconditions that correspond to static predicates are directly derived from the learning examples and therefore fewer programming actions are required. When static predicates are identified, the resulting compilation is also much more compact and produces smaller planning/instantiation times.

\begin{table*}
		\resizebox{\textwidth}{!}{%
		\begin{tabular}{l|l|l|l|l|l|l||l|l||l|l|l|l|l|l||l|l|}
& \multicolumn{8}{|c||}{\bf No Static}& \multicolumn{8}{|c|}{\bf Static}\\\cline{2-17}
& \multicolumn{2}{|c|}{\bf Pre} & \multicolumn{2}{|c|}{\bf Add} & \multicolumn{2}{|c|}{\bf Del} & \multicolumn{2}{|c||}{\bf}& \multicolumn{2}{|c|}{\bf Pre} & \multicolumn{2}{|c|}{\bf Add} & \multicolumn{2}{|c|}{\bf Del} & \multicolumn{2}{|c|}{\bf}\\ 			
			  & \multicolumn{1}{|c|}{\bf P} & \multicolumn{1}{|c|}{\bf R} & \multicolumn{1}{|c|}{\bf P} & \multicolumn{1}{|c|}{\bf R} & \multicolumn{1}{|c|}{\bf P} & \multicolumn{1}{|c|}{\bf R} &  \multicolumn{1}{|c|}{\bf P} & \multicolumn{1}{|c||}{\bf R}& \multicolumn{1}{|c|}{\bf P} & \multicolumn{1}{|c|}{\bf R} & \multicolumn{1}{|c|}{\bf P} & \multicolumn{1}{|c|}{\bf R} & \multicolumn{1}{|c|}{\bf P} & \multicolumn{1}{|c|}{\bf R} &  \multicolumn{1}{|c|}{\bf P} & \multicolumn{1}{|c|}{\bf R} \\
                          \hline
			Blocks & 1.0 & 1.0 & 1.0 & 1.0 & 1.0 & 1.0 & 1.0 & 1.0 & 1.0 & 1.0 & 1.0 & 1.0 & 1.0 & 1.0 & 1.0 & 1.0\\
			Driverlog & 1.0 & 0.36 & 0.75 & 0.86 & 1.0 & 0.71 & 0.92 & 0.64 & 0.9 & 0.64 & 0.56 & 0.71 & 0.86 & 0.86 & 0.78 & 0.73\\
			Ferry & 1.0 & 0.57 & 1.0 & 1.0 & 1.0 & 1.0 & 1.0 & 0.86 & 1.0 & 0.57 & 1.0 & 1.0 & 1.0 & 1.0 & 1.0 & 0.86\\
			Floortile & 0.52 & 0.68 & 0.64 & 0.82 & 0.83 & 0.91 & 0.66 & 0.80 & 0.68 & 0.68 & 0.89 & 0.73 & 1.0 & 0.82 & 0.86 & 0.74\\
			Grid & 0.62 & 0.47 & 0.75 & 0.86 & 0.78 & 1.0 & 0.71 & 0.78 & 0.79 & 0.65 & 1.0 & 0.86 & 0.88 & 1.0 & 0.89 & 0.83 \\
			Gripper & 1.0 & 0.67 & 1.0 & 1.0 & 1.0 & 1.0 & 1.0 & 0.89 & 1.0 & 0.67 & 1.0 & 1.0 & 1.0 & 1.0 & 1.0 & 0.89\\
			Hanoi & 1.0 & 0.50 & 1.0 & 1.0 & 1.0 & 1.0 & 1.0 & 0.83 & 0.75 & 0.75 & 1.0 & 1.0 & 1.0 & 1.0 & 0.92 & 0.92\\
			Miconic & 0.75 & 0.33 & 0.50 & 0.50 & 0.75 & 1.0 & 0.67 & 0.61 & 0.89 & 0.89 & 1.0 & 0.75 & 0.75 & 1.0 & 0.88 & 0.88\\
			Satellite & 0.60 & 0.21 & 1.0 & 1.0 & 1.0 & 0.75 & 0.87 & 0.65 & 0.82 & 0.64 & 1.0 & 1.0 & 1.0 & 0.75 & 0.94 & 0.80\\
			Transport & 1.0 & 0.40 & 1.0 & 1.0 & 1.0 & 0.80 & 1.0 & 0.73 & 1.0 & 0.70 & 0.83 & 1.0 & 1.0 & 0.80 & 0.94 & 0.83\\
			Visitall & 1.0 & 0.50 & 1.0 & 1.0 & 1.0 & 1.0 & 1.0 & 0.83 & 1.0 & 1.0 & 1.0 & 1.0 & 1.0 & 1.0 & 1.0 & 1.0\\
			Zenotravel & 1.0 & 0.36 & 1.0 & 1.0 & 1.0 & 0.71 & 1.0 & 0.69 &1.0 & 0.64 & 0.88 & 1.0 & 1.0 & 0.71 & 0.96 & 0.79\\
			\hline
			\bf  & 0.88 & 0.50 & 0.88 & 0.92 & 0.95 & 0.91 & 0.90 & 0.78 & 0.90 & 0.74 & 0.93 & 0.92 & 0.96 & 0.91 & 0.93 & 0.86\\
		\end{tabular}
	}
\caption{\small {\em Precision} and {\em recall} scores for learning tasks from labeled plans without (left) and with (right) static predicates.}
\label{tab:results_plans}
\end{table*}

\begin{table}
\begin{scriptsize}
	\begin{center}
		\begin{tabular}{l|c|c|c||c|c|c|}
                         & \multicolumn{3}{|c||}{\bf No Static}& \multicolumn{3}{|c|}{\bf Static}\\
			 & Total & Preprocess & Length  & Total & Preprocess &  Length\\
                         \hline
			Blocks & 0.04 & 0.00 & 72  & 0.03 & 0.00 & 72 \\
			Driverlog & 0.14 & 0.09 & 83 & 0.06 & 0.03 & 59 \\
			Ferry & 0.06 & 0.03 & 55 & 0.06 & 0.03 & 55 \\
			Floortile & 2.42 & 1.64 & 168 & 0.67 & 0.57 & 77 \\
			Grid & 4.82 & 4.75 & 88 & 3.39 & 3.35 & 72 \\
			Gripper & 0.03 & 0.01 & 43 & 0.01 & 0.00 & 43 \\
                        Hanoi & 0.12 & 0.06 & 48 & 0.09 & 0.06 & 39 \\
                        Miconic & 0.06 & 0.03 & 57 & 0.04 & 0.00 & 41 \\
			Satellite & 0.20 & 0.14 & 67 & 0.18 & 0.12 & 60 \\
			Transport & 0.59 & 0.53 & 61 & 0.39 & 0.35 & 48 \\
			Visitall & 0.21 & 0.15 & 40 & 0.17 & 0.15 & 36 \\
			Zenotravel & 2.07 & 2.04 & 71 & 1.01 & 1.00 & 55 \\			
		\end{tabular}
	\end{center}
        \end{scriptsize}
	\caption{\small Total planning time, preprocessing time and plan length for learning tasks from labeled plans without/with static predicates.}
	\label{tab:time_plans}	
\end{table}

\subsection{Learning from partially specified action models}

We evaluate now the ability of our approach to support partially specified action models; that is, addressing learning tasks of the kind $\Lambda''=\tup{\Psi,\Sigma,\Pi,\Xi_0}$. In this experiment, the model of half of the actions is given in $\Xi_0$ as an extra input of the learning task.

Tables~\ref{tab:results_plans_partial} and~\ref{tab:time_plans_partial} summarize the obtained results, which include the identification of static predicates. We only report the {\em precision} and {\em recall} of the {\em unknown} actions since the values of the metrics of the {\em known} action models is 1.0. In this experiment, a low value of {\em precision} or {\em recall} has a greater impact than in the corresponding $\Lambda'$ tasks because the evaluation is done only over half of the actions. This occurs, for instance, in the precondition \emph{recall} of domains such as {\em Floortile}, {\em Gripper} or {\em Satellite}.

Remarkably, the overall \emph{precision} is now $0.98$, which means that the contents of the learned models is highly reliable. The value of \emph{recall}, 0.87, is an indication that the learned models still miss some information (preconditions are again the component more difficult to be fully learned). Overall, the results confirm the previous trend: the more input knowledge of the task, the better the models and the less planning time. Additionally, the solution plans required for this task are smaller because it is only necessary to program half of the actions (the other half are included in the input knowledge $\Xi_0$). {\em Visitall} and {\em Hanoi} are excluded from this evaluation because they only contain one action schema.


\subsection{Learning from (initial,final) state pairs}
Finally, we evaluate our approach when input plans are not available and thereby the planner must not only compute the action models but also the plans that satisfy the input labels. Table~\ref{tab:results_states} and ~\ref{tab:time_states} summarize the results obtained for the task $\Lambda=\tup{\Psi,\Sigma,\Xi_0}$ using static predicates. Values for the {\em Zenotravel} and {\em Grid} domains are not reported because {\sc Madagascar} was not able to solve the corresponding planning tasks within a 1000 sec. time bound. The values of \emph{precision} and \emph{recall} are significantly lower than in Table ~\ref{tab:results_plans}. Given that the learning hypothesis space is now fairly under-constrained, actions can be reformulated and still be compliant with the inputs (e.g. the {\em blocksworld} operator {\small\tt stack} can be {\em learned} with the preconditions and effects of the {\small\tt unstack} operator and vice versa). We tried to minimize this effect with the additional input knowledge (static predicates and partially specified action models) and yet the results are below the scores obtained when learning from labeled plans.

\section{Related work}



Action model learning has also been studied in domains where there is partial or missing state observability. {\sf ARMS} works when no partial intermediate state is given. It defines a set of weighted constraints that must hold for the plans to be correct, and solves the weighted propositional satisfiability problem with a MAX-SAT solver~\cite{yang2007learning}. In order to efficiently solve the large MAX-SAT representations, {\sf ARMS} implements a hill-climbing method that models the actions approximately. In contrast to our model comparison validation which aims at covering all the training examples, ARMS maximizes the number of covered examples from a testing set.

{\sc SLAF} also deals with partial observability~\cite{amir:alearning:JAIR08}. Given a formula representing the initial belief state, a sequence of executed actions and the corresponding partially observed states, it builds a complete explanation of observations by models of actions through a CNF formula. The learning algorithm updates the formula of the belief state with every action and observation in the sequence and thus the final returned formula includes all consistent models. SLAF assesses the quality of the learned models with respect to the actual generative model.



{\sf LOCM} only requires the example plans as input without need for providing information about predicates or states~\cite{cresswell2013acquiring,cresswell2011generalised}. The lack of available information is overcome by exploiting assumptions about the kind of domain model it has to generate. Particularly, it assumes a domain consists of a collection of objects (sorts) whose defined set of states can be captured by a parameterized Finite State Machine. {\sf LOP} ({\sf LOCM} with Optimized Plans ~\cite{gregory2015domain})  incorporates static predicates and applies a post-processing step after the {\sf LOCM} analysis that requires a set of optimal plans to be used in the learning phase. This is done to mitigate the limitation of {\sf LOCM} of inducing similar models for domains with similar structures. {\sf LOP} compares the learned models with the corresponding reference model.

Compiling an action model learning task into classical planning is a general and flexible approach that allows to accommodate various amounts and kinds of input knowledge and opens up a path for addressing further learning and validation tasks. For instance, the example plans in $\Pi$ can be replaced or complemented by a set $\mathcal{O}$ of sequences of observations (i.e., fully or partial state observations with noisy or missing fluents~\cite{SohrabiRU16}), and learning tasks of the kind $\Lambda=\tup{\Psi,\Sigma,\mathcal{O},\Xi_0}$ would also be attainable. Furthermore, our approach seems extensible to learning other types of generative models (e.g. hierarchical models like HTN or behaviour trees) that can be more appealing than \strips\ models since they require less search effort to compute a a planning solution.


\begin{table}
\begin{footnotesize}
	\begin{center}
		\resizebox{\columnwidth}{!}{%
		\begin{tabular}{l|l|l|l|l|l|l||l|l|}
			 & \multicolumn{2}{|c|}{\bf Pre} & \multicolumn{2}{|c|}{\bf Add} & \multicolumn{2}{|c||}{\bf Del} & \multicolumn{2}{|c}{\bf}\\ \cline{2-9}			
			  & \multicolumn{1}{|c|}{\bf P} & \multicolumn{1}{|c|}{\bf R} & \multicolumn{1}{|c|}{\bf P} & \multicolumn{1}{|c|}{\bf R} & \multicolumn{1}{|c|}{\bf P} & \multicolumn{1}{|c||}{\bf R} &  \multicolumn{1}{|c|}{\bf P} & \multicolumn{1}{|c|}{\bf R} \\
			\hline
				Blocks & 1.0 & 1.0 & 1.0 & 1.0 & 1.0 & 1.0 & 1.0 & 1.0 \\
				Driverlog & 1.0 & 0.71 & 1.0 & 1.0 & 1.0 & 1.0 & 1.0 & 0.90 \\
				Ferry & 1.0 & 0.67 & 1.0 & 1.0 & 1.0 & 1.0 & 1.0 & 0.89 \\
				Floortile & 0.75 & 0.60 & 1.0 & 0.80 & 1.0 & 0.80 & 0.92 & 0.73 \\
                Grid & 1.0 & 0.67 & 1.0 & 1.0 & 1.0 & 1.0 & 0.84 & 0.78 \\
				Gripper & 1.0 & 0.50 & 1.0 & 1.0 & 1.0 & 1.0 & 1.0 & 0.83 \\
				Miconic & 1.0 & 1.0 & 1.0 & 1.0 & 1.0 & 1.0 & 1.0 & 1.0 \\
				Satellite & 1.0 & 0.57 & 1.0 & 1.0 & 1.0 & 1.0 & 1.0 & 0.86 \\
				Transport & 1.0 & 0.75 & 1.0 & 1.0 & 1.0 & 1.0 & 1.0 & 0.92 \\
				Zenotravel & 1.0 & 0.67 & 1.0 & 1.0 & 1.0 & 0.67 & 1.0 & 0.78 \\
				\hline
				\bf  & 0.98 & 0.71 & 1.0 & 0.98 & 1.0 & 0.95 & 0.98 & 0.87 \\
			\end{tabular}
		}
	\end{center}
\end{footnotesize}
\caption{\small {\em Precision} and {\em recall} scores for learning tasks with partially specified action models.}
\label{tab:results_plans_partial}
\end{table}

\begin{table}
\begin{footnotesize}
	\begin{center}
		\begin{tabular}{l|c|c|c|}			
			 & Total time & Preprocess & Plan length  \\
                         \hline
			Blocks & 0.07 & 0.01 & 54  \\
			Driverlog & 0.03 & 0.01 & 40 \\
			Ferry & 0.06 & 0.03 & 45 \\
			Floortile & 0.43 & 0.42 & 55 \\
                        Grid & 3.12 & 3.07 & 53 \\
			Gripper & 0.03 & 0.01 & 35 \\
			Miconic & 0.03 & 0.01 & 34  \\
			Satellite & 0.14 & 0.14 & 47 \\
			Transport & 0.23 & 0.21 & 37 \\
			Zenotravel & 0.90 & 0.89 & 40 \\
		\end{tabular}
	\end{center}
        \end{footnotesize}
	\caption{\small Time and plan length learning for learning tasks with partially specified action models.}
	\label{tab:time_plans_partial}	
\end{table}

\begin{table}
\begin{footnotesize}
	\begin{center}
		\resizebox{\columnwidth}{!}{%
		\begin{tabular}{l|l|l|l|l|l|l||l|l|}
			 & \multicolumn{2}{|c|}{\bf Pre} & \multicolumn{2}{|c|}{\bf Add} & \multicolumn{2}{|c||}{\bf Del} & \multicolumn{2}{|c}{\bf}\\ \cline{2-9}			
			  & \multicolumn{1}{|c|}{\bf P} & \multicolumn{1}{|c|}{\bf R} & \multicolumn{1}{|c|}{\bf P} & \multicolumn{1}{|c|}{\bf R} & \multicolumn{1}{|c|}{\bf P} & \multicolumn{1}{|c||}{\bf R} &  \multicolumn{1}{|c|}{\bf P} & \multicolumn{1}{|c|}{\bf R} \\
			\hline
            Blocks & 0.33 & 0.33 & 0.75 & 0.50 & 0.33 & 0.33 & 0.47 & 0.39 \\
            Driverlog & 1.0 & 0.29 & 0.33 & 0.67 & 1.0 & 0.50 & 0.78 & 0.48 \\
            Ferry & 1.0 & 0.67 & 0.50 & 1.0 & 1.0 & 1.0 & 0.83 & 0.89 \\
            Floortile & 0.67 & 0.40 & 0.50 & 0.40 & 1.0 & 0.40 & 0.72 & 0.40 \\
            Grid & - & - & - & - & - & - & - & - \\
            Gripper & 1.0 & 0.50 & 1.0 & 1.0 & 1.0 & 1.0 & 1.0 & 0.83 \\
            Miconic & 0.0 & 0.0 & 0.33 & 0.50 & 0.0 & 0.0 & 0.11 & 0.17 \\
            Satellite & 1.0 & 0.14 & 0.67 & 1.0 & 1.0 & 1.0 & 0.89 & 0.71 \\
            Transport & 0.0 & 0.0 & 0.25 & 0.5 & 0.0 & 0.0 & 0.08 & 0.17 \\
            Zenotravel & - & - & - & - & - & - & - & - \\
            \hline
            & 0.63 & 0.29 & 0.54 & 0.70 & 0.67 & 0.53 & 0.61 & 0.51 \\			
			\end{tabular}
		}
	\end{center}
\end{footnotesize}
\caption{\small {\em Precision} and {\em recall} scores for learning from (initial,final) state pairs.}
\label{tab:results_states}
\end{table}

\begin{table}
\begin{footnotesize}
	\begin{center}
		\begin{tabular}{l|c|c|c|}			
			 & Total time & Preprocess & Plan length  \\
			\hline
            Blocks & 2.14 & 0.00 & 58  \\
            Driverlog & 0.09 & 0.00 & 88 \\
            Ferry & 0.17 & 0.01 & 65 \\
            Floortile & 6.42 & 0.15 & 126 \\
            Grid & - & - & - \\
            Gripper & 0.03 & 0.00 & 47 \\
            Miconic & 0.04 & 0.00 & 68 \\
            Satellite & 4.34 & 0.10 & 126 \\
            Transport & 2.57 & 0.21 & 47 \\			
            Zenotravel & - & - & - \\
		\end{tabular}
	\end{center}
        \end{footnotesize}
	\caption{\small Time and plan length when learning from (initial,final) state pairs.}
	\label{tab:time_states}	
\end{table}

\section{Conclusions}
We presented a novel approach for learning \strips\ action models from examples using classical planning. The approach is flexible to various amounts of input knowledge and accepts partially specified action models. We also introduced the {\em precision} and {\em recall} metrics, widely used in ML, for evaluating the learned action models. These two metrics measure the soundness and completeness of the learned models and facilitate the identification of model flaws.

To the best of our knowledge, this is the first work on learning action models that is exhaustively evaluated over a wide range of domains and uses exclusively an {\em off-the-shelf} classical planner. The work in~\cite{SternJ17} proposes a planning compilation for learning action models from plan traces following the {\em finite domain} representation for the state variables. This is a theoretical study on the boundaries of the learned models and no experimental results are reported.

When example plans are available, we can compute accurate action models from small sets of learning examples (five examples per domain) in little computation time (less than a second). When action plans are not available, our approach still produces action models that are compliant with the input information. In this case, since learning is not constrained by actions, operators can be reformulated changing their semantics, in which case the comparison with a reference model turns out to be tricky.



Generating {\em informative} examples for learning planning action models is still an open issue. Planning actions include preconditions that are only satisfied by specific sequences of actions which have low probability of being chosen by chance~\cite{fern2004learning}. The success of recent algorithms for exploring planning tasks~\cite{FrancesRLG17} motivates the development of novel techniques that enable to autonomously collect informative learning examples. The combination of such exploration techniques with our learning approach is an appealing research direction that opens up the door to the bootstrapping of planning action models.

In many applications, the actual actions executed by the observed agent are not available but, instead, the resulting states can be observed. We plan to extend our approach for learning from state observations as it broadens the range of application to external observers and facilitates the representation of imperfect observability, as shown in plan recognition \cite{SohrabiRU16}, as well as learning from unstructured data, like state images \cite{AsaiF18}.

\subsection*{Acknowledgments}
This work is supported by the Spanish MINECO project TIN2017-88476-C2-1-R. Diego Aineto is partially supported by the {\it FPU16/03184} and Sergio Jim\'enez by the {\it RYC15/18009}, both programs funded by the Spanish government.

\bibliographystyle{aaai}
\bibliography{planlearnbibliography}

\begin{thebibliography}{}

\bibitem[\protect\citeauthoryear{Amir and Chang}{2008}]{amir:alearning:JAIR08}
Amir, E., and Chang, A.
\newblock 2008.
\newblock Learning partially observable deterministic action models.
\newblock {\em Journal of Artificial Intelligence Research} 33:349--402.

\bibitem[\protect\citeauthoryear{Asai and Fukunaga}{2018}]{AsaiF18}
Asai, M., and Fukunaga, A.
\newblock 2018.
\newblock Classical planning in deep latent space: Bridging the
  subsymbolic-symbolic boundary.
\newblock In {\em National Conference on Artificial Intelligence, {AAAI-18}}.

\bibitem[\protect\citeauthoryear{Bonet, Palacios, and
  Geffner}{2009}]{bonet2009automatic}
Bonet, B.; Palacios, H.; and Geffner, H.
\newblock 2009.
\newblock {A}utomatic {D}erivation of {M}emoryless {P}olicies and
  {F}inite-{S}tate {C}ontrollers {U}sing {C}lassical {P}lanners.
\newblock In {\em International Conference on Automated Planning and
  Scheduling, {ICAPS-09}}.

\bibitem[\protect\citeauthoryear{Cresswell and
  Gregory}{2011}]{cresswell2011generalised}
Cresswell, S., and Gregory, P.
\newblock 2011.
\newblock {G}eneralised {D}omain {M}odel {A}cquisition from {A}ction {T}races.
\newblock In {\em International Conference on Automated Planning and
  Scheduling, {ICAPS-11}}.

\bibitem[\protect\citeauthoryear{Cresswell, McCluskey, and
  West}{2013}]{cresswell2013acquiring}
Cresswell, S.~N.; McCluskey, T.~L.; and West, M.~M.
\newblock 2013.
\newblock Acquiring planning domain models using {LOCM}.
\newblock {\em The Knowledge Engineering Review} 28(02):195--213.

\bibitem[\protect\citeauthoryear{Davis and
  Goadrich}{2006}]{davis2006relationship}
Davis, J., and Goadrich, M.
\newblock 2006.
\newblock The {R}elationship {B}etween {P}recision-{R}ecall and {ROC} {C}urves.
\newblock In {\em International Conference on Machine Learning},  233--240.
\newblock ACM.

\bibitem[\protect\citeauthoryear{Fern, Yoon, and
  Givan}{2004}]{fern2004learning}
Fern, A.; Yoon, S.~W.; and Givan, R.
\newblock 2004.
\newblock {L}earning {D}omain-{S}pecific {C}ontrol {K}nowledge from {R}andom
  {W}alks.
\newblock In {\em International Conference on Automated Planning and
  Scheduling, {ICAPS-04}},  191--199.

\bibitem[\protect\citeauthoryear{Fikes and Nilsson}{1971}]{fikes1971strips}
Fikes, R.~E., and Nilsson, N.~J.
\newblock 1971.
\newblock {STRIPS}: {A} {N}ew {A}pproach to the {A}pplication of {T}heorem
  {P}roving to {P}roblem {S}olving.
\newblock {\em Artificial Intelligence} 2(3-4):189--208.

\bibitem[\protect\citeauthoryear{Fox and Long}{1998}]{fox:TIM:JAIR1998}
Fox, M., and Long, D.
\newblock 1998.
\newblock {T}he {A}utomatic {I}nference of {S}tate {I}nvariants in {TIM}.
\newblock {\em Journal of Artificial Intelligence Research} 9:367--421.

\bibitem[\protect\citeauthoryear{Fox and Long}{2003}]{fox2003pddl2}
Fox, M., and Long, D.
\newblock 2003.
\newblock {PDDL2.1}: {A}n {E}xtension to {PDDL} for {E}xpressing {T}emporal
  {P}lanning {D}omains.
\newblock {\em Journal of Artificial Intelligence Research} 20:61--124.

\bibitem[\protect\citeauthoryear{Franc{\`{e}}s \bgroup et al\mbox.\egroup
  }{2017}]{FrancesRLG17}
Franc{\`{e}}s, G.; Ram{\'{\i}}rez, M.; Lipovetzky, N.; and Geffner, H.
\newblock 2017.
\newblock {P}urely {D}eclarative {A}ction {D}escriptions are {O}verrated:
  {C}lassical {P}lanning with {S}imulators.
\newblock In {\em International Joint Conference on Artificial Intelligence,
  {IJCAI-17}},  4294--4301.
\newblock AAAI Press.

\bibitem[\protect\citeauthoryear{Geffner and Bonet}{2013}]{geffner:book:2013}
Geffner, H., and Bonet, B.
\newblock 2013.
\newblock {\em {A} {C}oncise {I}ntroduction to {M}odels and {M}ethods for
  {A}utomated {P}lanning}.
\newblock Morgan \& Claypool Publishers.

\bibitem[\protect\citeauthoryear{Ghallab, Nau, and
  Traverso}{2004}]{ghallab2004automated}
Ghallab, M.; Nau, D.; and Traverso, P.
\newblock 2004.
\newblock {\em Automated Planning: Theory and Practice}.
\newblock Elsevier.

\bibitem[\protect\citeauthoryear{Gregory and
  Cresswell}{2015}]{gregory2015domain}
Gregory, P., and Cresswell, S.
\newblock 2015.
\newblock {D}omain {M}odel {A}cquisition in the {P}resence of {S}tatic
  {R}elations in the {LOP} {S}ystem.
\newblock In {\em International Conference on Automated Planning and
  Scheduling, {ICAPS-15}},  97--105.

\bibitem[\protect\citeauthoryear{Howey, Long, and Fox}{2004}]{howey2004val}
Howey, R.; Long, D.; and Fox, M.
\newblock 2004.
\newblock {VAL}: {A}utomatic {P}lan {V}alidation, {C}ontinuous {E}ffects and
  {M}ixed {I}nitiative {P}lanning {U}sing {PDDL}.
\newblock In {\em 16th IEEE International Conference on Tools with Artificial
  Intelligence, {ICTAI 2004}},  294--301.
\newblock IEEE.

\bibitem[\protect\citeauthoryear{Jim{\'e}nez \bgroup et al\mbox.\egroup
  }{2012}]{jimenez2012review}
Jim{\'e}nez, S.; De~La~Rosa, T.; Fern{\'a}ndez, S.; Fern{\'a}ndez, F.; and
  Borrajo, D.
\newblock 2012.
\newblock A review of machine learning for automated planning.
\newblock {\em The Knowledge Engineering Review} 27(04):433--467.

\bibitem[\protect\citeauthoryear{Kambhampati}{2007}]{kambhampati:modellite:AAAI2007}
Kambhampati, S.
\newblock 2007.
\newblock {M}odel-lite {P}lanning for the {W}eb {A}ge {M}asses: {T}he
  {C}hallenges of {P}lanning with {I}ncomplete and {E}volving {D}omain
  {M}odels.
\newblock In {\em National Conference on Artificial Intelligence, {AAAI-07}},
  1601--1605.

\bibitem[\protect\citeauthoryear{McDermott \bgroup et al\mbox.\egroup
  }{1998}]{mcdermott1998pddl}
McDermott, D.; Ghallab, M.; Howe, A.; Knoblock, C.; Ram, A.; Veloso, M.; Weld,
  D.; and Wilkins, D.
\newblock 1998.
\newblock {PDDL} -- {T}he {P}lanning {D}omain {D}efinition {L}anguage.

\bibitem[\protect\citeauthoryear{Michalski, Carbonell, and
  Mitchell}{2013}]{michalski2013machine}
Michalski, R.~S.; Carbonell, J.~G.; and Mitchell, T.~M.
\newblock 2013.
\newblock {\em Machine learning: An artificial intelligence approach}.
\newblock Springer Science \& Business Media.

\bibitem[\protect\citeauthoryear{Muise}{2016}]{muise2016planning}
Muise, C.
\newblock 2016.
\newblock Planning. domains.
\newblock {\small
  \url{http://bibbase.org/network/publication/muise-planningdomains/}}.

\bibitem[\protect\citeauthoryear{Ram{\'\i}rez}{2012}]{ramirez2012plan}
Ram{\'\i}rez, M.
\newblock 2012.
\newblock {\em {P}lan {R}ecognition as {P}lanning}.
\newblock Ph.D. Dissertation, Universitat Pompeu Fabra.

\bibitem[\protect\citeauthoryear{Rintanen}{2014}]{rintanen2014madagascar}
Rintanen, J.
\newblock 2014.
\newblock {M}adagascar: {S}calable {P}lanning with {SAT}.
\newblock {\em Proceedings of the 8th International Planning Competition,
  {IPC-2014}}.

\bibitem[\protect\citeauthoryear{Segovia-Aguas, Jim{\'e}nez, and
  Jonsson}{2016}]{segovia2016hierarchical}
Segovia-Aguas, J.; Jim{\'e}nez, S.; and Jonsson, A.
\newblock 2016.
\newblock {H}ierarchical {F}inite {S}tate {C}ontrollers for {G}eneralized
  {P}lanning.
\newblock In {\em International Joint Conference on Artificial Intelligence,
  {IJCAI-16}},  3235--3241.
\newblock AAAI Press.

\bibitem[\protect\citeauthoryear{Segovia-Aguas, Jim{\'e}nez, and
  Jonsson}{2017}]{segovia2017generating}
Segovia-Aguas, J.; Jim{\'e}nez, S.; and Jonsson, A.
\newblock 2017.
\newblock {G}enerating {C}ontext-{F}ree {G}rammars using {C}lassical
  {P}lanning.
\newblock In {\em International Joint Conference on Artificial Intelligence,
  {ICAPS-17}},  4391--4397.

\bibitem[\protect\citeauthoryear{Slaney and
  Thi{\'e}baux}{2001}]{slaney2001blocks}
Slaney, J., and Thi{\'e}baux, S.
\newblock 2001.
\newblock {B}locks {Wo}rld {R}evisited.
\newblock {\em Artificial Intelligence} 125(1-2):119--153.

\bibitem[\protect\citeauthoryear{Sohrabi, Riabov, and
  Udrea}{2016}]{SohrabiRU16}
Sohrabi, S.; Riabov, A.~V.; and Udrea, O.
\newblock 2016.
\newblock {P}lan {R}ecognition as {P}lanning {R}evisited.
\newblock In {\em International Joint Conference on Artificial Intelligence,
  {IJCAI-16}},  3258--3264.
\newblock AAAI Press.

\bibitem[\protect\citeauthoryear{Stern and Juba}{2017}]{SternJ17}
Stern, R., and Juba, B.
\newblock 2017.
\newblock {E}fficient, {S}afe, and {P}robably {A}pproximately {C}omplete
  {L}earning of {A}ction {M}odels.
\newblock In {\em International Joint Conference on Artificial Intelligence,
  {IJCAI-17}},  4405--4411.
\newblock AAAI Press.

\bibitem[\protect\citeauthoryear{Vallati \bgroup et al\mbox.\egroup
  }{2015}]{vallati:IPC:AIM2015}
Vallati, M.; Chrpa, L.; Grzes, M.; McCluskey, T.~L.; Roberts, M.; and Sanner,
  S.
\newblock 2015.
\newblock The 2014 {I}nternational {P}lanning {C}ompetition: {P}rogress and
  {T}rends.
\newblock {\em AI Magazine} 36(3):90--98.

\bibitem[\protect\citeauthoryear{Yang, Wu, and Jiang}{2007}]{yang2007learning}
Yang, Q.; Wu, K.; and Jiang, Y.
\newblock 2007.
\newblock Learning action models from plan examples using weighted {MAX-SAT}.
\newblock {\em Artificial Intelligence} 171(2-3):107--143.

\end{thebibliography}
\end{document}